\documentclass{llncs}
\usepackage[utf8]{inputenc}
\usepackage[T1]{fontenc}
\usepackage{amsmath}
\usepackage{amssymb}
\usepackage{lmodern}
\usepackage{mathrsfs}
\usepackage{stmaryrd}
\usepackage[usenames,dvipsnames,svgnames,table]{xcolor}
\usepackage[unicode,pdfmenubar,linktoc=all,hidelinks,bookmarks, pdftex]{hyperref}
\usepackage{tikz}
\usepackage[ruled,vlined]{algorithm2e}
\usetikzlibrary{fadings}
\usetikzlibrary{shapes,decorations,positioning}
\usetikzlibrary{fadings,shapes,decorations,positioning,calc}
\usetikzlibrary{decorations.markings,decorations.pathreplacing}
\usetikzlibrary{shapes.geometric,automata,positioning}
\usetikzlibrary{shadows}
\usetikzlibrary{calc}
\usetikzlibrary{decorations,decorations.pathmorphing,decorations.pathreplacing}
\usetikzlibrary{calc}
\usepackage{xspace}
\usepackage{graphicx}
\usepackage{enumerate}
\usepackage{csquotes}
\usepackage{xparse}
\usepackage{cite}
\usepackage{booktabs}
\usepackage[kerning,spacing]{microtype}

\DeclareDocumentCommand{\todo}{o g}{\IfNoValueTF{#1}{\begingroup\color{magenta}TODO: #2\endgroup}{\begingroup\color{magenta}#1 #2\endgroup}}

\newcommand{\ol}[1]{\overline{#1}}

\renewcommand{\leq}{\leqslant}

\renewcommand{\geq}{\geqslant}

\newcommand{\cspMIof}[2]{\ensuremath{\mathit{CR}^{#2}(#1)}} %

\newcommand{\OCFsolutionsOf}[1]{\ensuremath{\mathit{Sol_{OCF}}(#1)}}

\newcommand{\OCFsolutionsRnMI}[1]{\ensuremath{\OCFsolutionsOf{\cspMIof{\R_n}{n-1}}}}

\newcommand{\induzierteOCF}[1]{\ensuremath{\kappa_{\!#1}}} %
\newcommand{\syntheticKB}[1]{\texttt{kb\_synth<\(n\)>\_c<\(2n\!\!-\!\!1\)>.pl}}

\newcommand{\R}{\ensuremath{\mathcal R}}

\newcommand{\notA}{\overline{A}}

\newcommand{\notB}{\overline{B}}

\newcommand{\naturals}{\mathbb{N}}

\makeatletter
\ifx\centernot\undefined
\newcommand*{\centernot}{%
	\mathpalette\@centernot
}
\fi
\def\@centernot#1#2{%
	\mathrel{%
		\rlap{%
			\settowidth\dimen@{$\m@th#1{#2}$}%
			\kern.5\dimen@
			\settowidth\dimen@{$\m@th#1=$}%
			\kern-.5\dimen@
			$\m@th#1\not$%
		}%
		{#2}%
	}%
}
\makeatother
\DeclareRobustCommand\nmableitSymb{\mathrel|\mkern-.5mu\joinrel\sim} %
\newcommand{\nmableit}{\ensuremath{\mbox{$\,\nmableitSymb\,$}}} %

\newcommand{\beweisendezeichen}%
{\penalty50\hspace*{0pt plus 1fil}\parfillskip=0pt\mbox{$\Box$}}

\newcommand{\fussnoteOhneMarkierung}[1]%
{%
\footnote{#1}%
\addtocounter{footnote}{-1}%
}

\newlength{\abstand}
\makeatletter
\newcommand{\ifLatexThree}[2]{\@ifpackageloaded{xparse}{#1}{#2}}
\newcommand{\ifAMSmath}[2]{\@ifpackageloaded{amsmath}{#1}{#2}}
\newcommand{\ifMathSCR}[2]{\@ifpackageloaded{mathrsfs}{#1}{#2}}
\newcommand{\ifMathHyperREF}[2]{\@ifpackageloaded{hyperref}{#1}{#2}}
\makeatother

\ifLatexThree{%
	\NewDocumentCommand{\headword}{s o m}{\IfBooleanTF{#1}{#3}{\textbf{#3}}\IfNoValueTF{#2}{\index{#3}}{\index{#2}}}%
}{%
	\makeatletter
	\def\@headword#1{\textbf{#1}\index{#1}}%
	\def\@@headword#1{#1\index{#1}}%
	\def\headword#1{\@ifstar\@headword{#1}\@@headword{#1}}%
	\makeatother
}

\makeatletter
\@ifundefined{naturals}{\newcommand{\naturals}{\mathbb{N}}}{}
\@ifundefined{ol}{\newcommand{\ol}[1]{\overline{#1}}}{}
\makeatother

\ifx\nmableit\undefined
\makeatletter
\ifx\centernot\undefined
\newcommand*{\centernot}{%
	\mathpalette\@centernot
}
\fi
\def\@centernot#1#2{%
	\mathrel{%
		\rlap{%
			\settowidth\dimen@{$\m@th#1{#2}$}%
			\kern.5\dimen@
			\settowidth\dimen@{$\m@th#1=$}%
			\kern-.5\dimen@
			$\m@th#1\not$%
		}%
		{#2}%
	}%
}
\makeatother
\DeclareRobustCommand\nmableitSymb{\mathrel|\mkern-.5mu\joinrel\sim} %
\newcommand{\nmableit}{\ensuremath{\mbox{$\,\nmableitSymb\,$}}} %
\fi

\makeatletter
\ifMathHyperREF{%
	\newcommand{\seqref}[1]{\hyperref[{#1}]{\textup{\tagform@split{\getrefnumber{#1}}}}}%
}{%
	\newcommand{\seqref}[1]{\textup{\tagform@split{\getrefnumber{#1}}}}%
}
\newcommand\tagform@split[1]{%
	\begingroup
	\m@th\normalfont(\ignorespaces #1\unskip\@@italiccorr)%
	\endgroup
}
\makeatother

\makeatletter
\newcommand{\leqnomode}{\tagsleft@true\let\veqno\@@leqno}
\newcommand{\reqnomode}{\tagsleft@false\let\veqno\@@eqno}
\newcommand{\pushright}[1]{\ifmeasuring@#1\else\omit\hfill$\displaystyle#1$\fi\ignorespaces}
\newcommand{\pushleft}[1]{\ifmeasuring@#1\else\omit$\displaystyle#1$\hfill\fi\ignorespaces}
\newcommand{\specialcell}[1]{\ifmeasuring@#1\else\omit$\displaystyle#1$\ignorespaces\fi}

\makeatother

\newcommand{\tuple}[1]{\ensuremath{\langle{#1}\rangle}}

\newcommand{\modelsOf}[1]{\ensuremath{\llbracket #1\rrbracket}}

\newcommand{\beliefsOf}[1]{\ensuremath{\mathit{Bel}\left(#1\right)}}
\newcommand{\beliefsOfCond}[1]{\ensuremath{\mathit{Bel}^\mathit{cond}\left(#1\right)}}

\newcommand{\propLang}{\ensuremath{\mathcal{L}}}

\newcommand{\condLang}{\ensuremath{(\propLang|\propLang)}}

\newcommand{\ramseyCond}[2]{\ensuremath{(\,#1\,|\,#2\,)}}

\ifMathSCR{%
}{%
}

\title{Descriptor Revision for Conditionals:\\ Literal Descriptors and Conditional Preservation}

\newcommand{\disjunctionFree}{elementary}

\def\myorcidID#1{\ifx\hyper@anchor\@undefined\unskip$^{[#1]}$\else\href{https://orcid.org/#1}{\protect\includegraphics[height=8px]{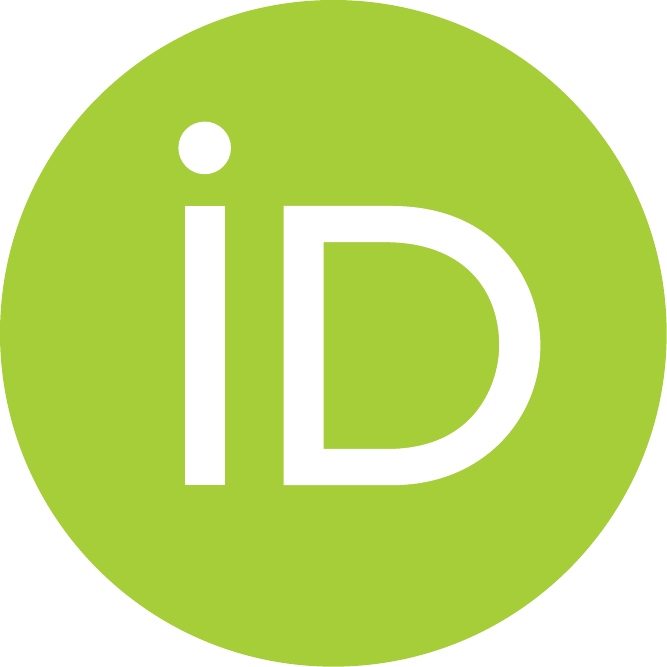}}\fi}

\author{%
	Kai Sauerwald%
	\myorcidID{0000-0002-1551-7016}
	\and Jonas Haldimann%
	\myorcidID{0000-0002-2618-8721}
	\and Martin von Berg
	\and Christoph Beierle
}
\authorrunning{Beierle and Sauerwald} %
\institute{\texttt{\{%
kai.sauerwald,jonas.haldimann,christoph.beierle\}@fernuni-hagen.de}\\FernUniversität in Hagen\\58084 Hagen, Germany%
}

\begin{document}
	
\maketitle

\begin{abstract}
Descriptor revision by Hansson is a framework for addressing the
problem of belief change.
In descriptor revision, different kinds of change processes are
dealt with in a joint framework. Individual change requirements
are qualified by specific success conditions expressed by a belief descriptor,
and belief descriptors can be combined by logical connectives.
This is in contrast to the currently dominating AGM paradigm shaped by
Alchourrón, Gärdenfors, and Makinson, where different kinds of changes,
like a revision or a contraction, are dealt with separately.
In this article, we investigate the realisation of descriptor revision
for a conditional logic while restricting descriptors to the conjunction of literal descriptors.
We apply the principle of conditional preservation developed by Kern-Isberner
 to descriptor revision
for conditionals, show how descriptor revision for conditionals under these
restrictions can be characterised by a constraint satisfaction problem,
and implement it using constraint logic programming.
Since our conditional logic subsumes propositional logic, our approach
also realises descriptor revision for propositional logic. 
\end{abstract} 
\section{Introduction}
\label{sec:introduction}
The approach to belief change by Alchourrón, Gärdenfors, and Makinson (AGM) and its successors are currently the dominating paradigm for belief change. 
In this theory, three main kinds of belief changes are subject of interest: \emph{revision} (incorporating new beliefs into an agent's belief state while maintaining consistency), \emph{contraction} (removing beliefs from the agent's belief state), and \emph{expansion} (incorporating new beliefs into an agent's belief state without maintaining consistency). 
The most prominent difference between these kinds of changes is their success condition. 
The overall approach to the problem of belief change by AGM is top-down, 
starting from the axiomatisation of each of the three kinds of changes and then investigating the representational issues through representation theorems.

In the last 20 years, the AGM theory has been extended into several directions and has been deeply investigated. This gives new insights on the requirements of representation and conceptual problems of (AGM) belief change. 
In particular, for Hansson \cite{KS_Hansson2019}, the requirement of epistemic states for iterative belief change \cite{KS_DarwichePearl1997}, the central role of conditionals in belief change and non-monotonic logic \cite{KS_MakinsonGaerdenfors1991,KS_SauerwaldKern-IsbernerBeierle2020} and 
problems like the non-finite representability of the result of a contraction \cite{KS_Hansson2017a} or concerns about the \enquote{select-and-intersect} approach of AGM\cite{KS_Hansson2019} were a motivation to design a new framework for belief change. 
Descriptor revision by Hansson \cite{KS_Hansson2014} follows the top-down approach to belief change, but, in contrast to the AGM paradigm, in descriptor revision, different kinds of changes are expressible in one joint framework. 
For this, Hansson introduced a language for success conditions, called belief descriptors,
implying that
the success condition of a change is not implicitly hidden in the kind of operation, but an explicit part of the change process.
This allows to express and analyse change processes that go beyond the classical AGM operations, e.g., a change process where a contraction of a belief $ \alpha $ and a revision by $ \beta $ appear at the same time.
Descriptor revision has been broadly investigated by Hansson \cite{KS_Hansson2017,KS_Hansson2015,KS_Hansson2016,KS_Hansson2016a,KS_Hansson2017a,KS_Hansson2019}, but did not gain as much attention as AGM \cite{KS_Zhang2017}. 
In particular, to the best of our knowledge, until now, no approach to the realisation of descriptor revision is available.

In this article, we investigate descriptor revision for a conditional logic while using ordinal conditional functions \cite{KS_Spohn1988}, also called ranking functions, as representation for epistemic states.
We outline how to instantiate the framework of descriptor revision for this logic and design an approach for its realisation.
Furthermore, for descriptor revision we use and adapt the sophisticated principle of conditional preservation by Kern-Isberner \cite{Kern-Isberner01f,KS_Kern-Isberner2004} for ranking functions.
In summary, the main contributions of this article are:
\begin{itemize}
	\item Introduction of conditional descriptor revision, which introduces the principle of conditional preservation to the framework of descriptor revision.
	\item A sound and complete characterisation of conditional descriptor revision for elementary descriptors by a constraint satisfaction problem.
	\item Implementation of elementary descriptor revision using constraint logic programming and by employing the developed characterisation.
\end{itemize}

The article is organised as follows. In Section \ref{sec:prelim}, we present logical preliminaries. We recall descriptors and descriptor revision in Section \ref{sec:prelim_descriptor_revision}.
Section \ref{sec:conditional_descriptor_revision} introduces our framework of conditional descriptor revision. 
Section \ref{sec:constraint_system} develops a characterisation of conditional descriptor revision for elementary descriptors by a constraint satisfaction problem. The implementation of this approach is sketched in Section \ref{sec:implementation}. 
We conclude and point out future work in Section \ref{sec:conclusion}. 
\section{Logical Preliminaries}
\label{sec:prelim}
\renewcommand{\propLang}{\ensuremath{\mathcal{L}^\mathrm{prop}}}
\renewcommand{\condLang}{\ensuremath{\mathcal{L}^\mathrm{cond}}}
\renewcommand{\modelsOf}[1]{\ensuremath{\mathit{Mod}(#1)}}
\renewcommand{\modelsOf}[1]{\ensuremath{\mathit{Mod}(#1)}}
\renewcommand{\ramseyCond}[2]{\ensuremath{(#1|#2)}}

Let $ \Sigma $ be a propositional signature (non empty finite set of propositional variables) and $ \propLang $ the propositional language over $ \Sigma $. 
With upper case letters $ A,B,C,\ldots $, we denote formulas in $ \propLang $ and with lower case letters $ a,b,c,\ldots$ propositional variables from $ \Sigma $.
	We allow the typical abbreviation $ A \to B $ for $ \neg A \lor B $,
	 abbreviate $ A \land B $ by $ AB $ and write $ \ol{A} $ for $ \neg A $.
	 With $ \top  $, we denote a propositional tautology and with $ \bot $ a propositional falsum.
	The set of 
	propositional interpretations $ \Omega=\mathcal{P}(\Sigma) $, also called set of worlds, is identified with the set of corresponding complete conjunctions over $ \Sigma $, where $ \mathcal{P}(\cdot) $ is the powerset operator.
Propositional entailment is denoted by $ \models $, the set of models of $ A $ with $ \modelsOf{A} $, and $ Cn(A)=\{ B \mid A \models B \} $ is the deductive closure of $ A $.
For a set $ X $, we define $ Cn(X)=\{ B \mid X \models B \} $ and say $ X $ is deductively closed if $ X = Cn(X) $.
In the context of belief change, a deductively closed set is also called a \emph{belief set}.

A function  $\kappa: \Omega \to \naturals$ such that $\kappa^{-1}(0) \neq \emptyset$ is a called a \emph{ordinal conditional function (OCF)}, also called a \emph{ranking
	function} \cite{KS_Spohn1988}. It expresses degrees of plausibility of interpretations.
This is lifted to propositional formulas $A$
by specifying degrees of disbelief.
More formally, we have
$\kappa(A) := \min\{\kappa(\omega) \mid \omega \models A\}$, so that
$\kappa(A \vee B) = \min\{\kappa(A), \kappa(B)\}$. 
With $ \modelsOf{\kappa} =\{\omega\mid\kappa(\omega)=0\}$, we denote the minimal interpretations with respect to $\kappa$, and $ \beliefsOf{\kappa} $ denotes the theory of propositional formulas that hold in all $\omega \in \modelsOf{\kappa} $.

Over $ \Sigma $ and $ \propLang $, we define the set of conditionals $ \condLang = \{ (B|A) \mid A,B \in \mathcal{L}\}$. A conditional $(B|A)$ formalizes ``\textit{if $A$ then usually $B$}'' and establishes a plausible connection between the \emph{antecedent} $A$ and the \emph{consequent} $ B $. 
Conditionals with tautological antecedents are taken as plausible statements about the world. 
Because conditionals go well beyond classical logic, they require a richer setting for their semantics than classical logic.
Following De Finetti \cite{KS_Finetti1937}, a conditional $(B|A)$ can be \emph{verified} (\emph{falsified}) by a possible world $\omega$ iff $\omega \models AB$ ($\omega \models A \overline{B}$). 
If $\omega \not\models A$, then we say the conditional is \emph{not applicable} to $\omega$.

Ranking functions serve here as interpretations in a model theory for a conditional logic. 
We say a conditional $(B|A)$ is accepted in a ranking function $\kappa$,
written as $
\kappa \models (B|A) $, iff $
\kappa(AB) < \kappa(A\notB), $
i.e., iff the verification $AB$ of the conditional is more plausible than its falsification $A\notB$. 
For a propositional formula $A$, we define $\kappa \models A$ if $\kappa \models (A|\top)$, i.e., iff $\kappa(A) < \kappa(\notA)$ or iff $\kappa(\notA) > 0$, since at least one of $\kappa(A), \kappa(\notA)$ must be $0$ due to $\kappa^{-1}(0) \neq \emptyset$.
The models of a conditional $ \ramseyCond{B}{A} $ are the set of all ranking functions accepting $ \ramseyCond{B}{A} $, i.e. $ \modelsOf{\ramseyCond{B_1}{A_1}} = \{ \kappa \mid \kappa\models\ramseyCond{B}{A} \} $.
A conditional $ \ramseyCond{B_1}{A_1} $ entails $ \ramseyCond{B_2}{A_2} $, written $ \ramseyCond{B_1}{A_1} \models \ramseyCond{B_2}{A_2} $, if $ \modelsOf{\ramseyCond{B_1}{A_1}} \subseteq \modelsOf{\ramseyCond{B_2}{A_2}} $ holds.
Furthermore, we define the set of consequences for $ X\subseteq\condLang $ by $ Cn(X) = \{  \ramseyCond{B}{A} \mid X \models \ramseyCond{B}{A} \} $. 
As usual, $ X\subseteq\condLang $ is called deductively closed if $ X = Cn(X) $.
This ranking function based semantics can be mapped to, and can also be obtained from, other semantics of conditionals \cite{KS_BeierleKern-Isberner2012}.

\newcommand{\Kpen}{\ensuremath{\kappa_\mathit{p}}}
\newcommand{\Kpcp}{\ensuremath{\kappa^\circ_\mathit{p}}}
\newcommand{\Kpsi}{\ensuremath{\kappa^\circ_\Psi}}
\newcommand{\Rpen}{\ensuremath{\mathcal{R}_\mathit{pen}}}
\begin{example}[adapted \cite{BeierleKernIsbernerSauerwaldBockRagni2019KIzeitschrift}]\label{exmpl:abstract-framework}
	Let $\Sigma=\{p,b,f\}$ 
	with $p$ meaning 
	\enquote{penguin}, $b$ \enquote{bird} and $f$ \enquote{able to fly}.
	\enquote{Birds normally fly} is modelled with the conditional $r_1=(f|b)$, \enquote{penguins normally do not fly} with $r_2=(\ol{f}|p)$, and \enquote{penguins are normally birds} with $r_3=(b|p)$.
	Consider the ranking function $ \Kpen $ from Table \ref{tbl:exmpl}, which will act as our running example for the following sections (where we will also elaborate the other ranking function and conditionals shown in Table~\ref{tbl:exmpl}).
	Table \ref{tbl:exmpl} also contains the verifying and falsifying interpretations of the conditional $ (\ol{f}|p) $.
	The ranking function $ \Kpen $ 
	accepts all conditionals in $\Rpen = \{r_1, r_2, r_3 \}$, i.e. $ \Kpen\models r_i $ for all $ 1\leq i\leq 3 $.
	For example, $\kappa \models r_1$ because $\kappa(bf) = 0 < 1 = \kappa(b\ol{f})$ holds.
	Sure, after reading a lot of papers from knowledge representation, the ranking function $ \Kpen $ is the only viable belief state representing beliefs about penguins, flying and birds for an agent.
\end{example}

\begin{table*}[t]
	\centering
	\newcolumntype{x}[1]{>{\centering\arraybackslash\hspace{0pt}}p{#1}}
\begin{tabular}{x{1cm}|ccc|cc}
	\toprule
	                         &                     \multicolumn{3}{c|}{conditionals}                      & \multicolumn{2}{c}{belief states} \\
	$\omega$                 & $ \ramseyCond{p}{b} $ & $ \ramseyCond{f}{p} $ & $ \ramseyCond{\ol{f}}{p} $ & $\Kpen(\omega)$ & $\Kpcp(\omega)$ \\ \midrule
	$b\,f\,p$                &           v           &           v           &             f              &        2        &        1        \\
	$b\,f\,\ol{p}$           &           f           &                       &                            &        0        &        2        \\
	$b\,\ol{f}\,p$           &           v           &           f           &             v              &        1        &        1        \\
	$b\,\ol{f}\,\ol{p}$      &           f           &                       &                            &        1        &        3        \\
	$\ol{b}f\,p$             &                       &           v           &             f              &        4        &        3        \\
	$\ol{b}f\,\ol{p}$        &                       &                       &                            &        0        &        0        \\
	$\ol{b}\,\ol{f}\,p$      &                       &           f           &             v              &        2        &        2        \\
	$\ol{b}\,\ol{f}\,\ol{p}$ &                       &                       &                            &        0        &        0        \\ \bottomrule
\end{tabular}
	\caption{Verifying (v) and falsifying (f) interpretations for the conditionals $ \ramseyCond{p}{b}$, $\ramseyCond{f}{p}$, and $\ramseyCond{\ol{f}}{p} $,   and the ranking functions for the running penguin example.}
	\label{tbl:exmpl}
\end{table*} 
\section{Descriptors and Descriptor Revision}
\label{sec:prelim_descriptor_revision}
\newcommand{\belief}[1]{\ensuremath{\mathfrak{B}#1}}

The main building blocks of descriptor revision are belief descriptors, which provide a language for expressing membership constraints for a belief set.

\begin{definition}[Descriptor \cite{KS_Hansson2017}]\label{def:descriptor}
Let $ \mathcal{L} $ be a logical language. 
For any sentence $ \varphi\in\mathcal{L} $ the expression $ \belief{\varphi} $ is an \emph{atomic descriptor} (over $ \mathcal{L} $). 
Any connection of atomic descriptors with disjunction, conjunction and negation is called a \emph{molecular descriptor} (over $ \mathcal{L} $). 
A \emph{composite descriptor} (over $ \mathcal{L} $) is a set of molecular descriptors (over $ \mathcal{L} $).
\end{definition}

Like Hansson \cite{KS_Hansson2017}, we simplify notions by denoting composite descriptors just as descriptors. We also call a molecular descriptor of the form $ \belief{\varphi} $ or $ \neg\belief{\varphi} $ \emph{literal descriptor}. An  \emph{\disjunctionFree\ descriptor} is a set of literal descriptors.

\begin{definition}[Descriptor semantics \cite{KS_Hansson2017}]
	An atomic descriptor $ \belief{\varphi} $ holds in a belief set $ X $, written $ X \Vdash  \belief{\varphi} $, if $ \varphi \in X $. 
	This is lifted to molecular descriptors truth-functionally. 
	A descriptor $ \Psi $ holds in $ X $, likewise written $ X \Vdash \Psi $, if $ X \Vdash \alpha $ holds for every molecular descriptor $ \alpha\in \Psi $.
\end{definition}

For an example of descriptors, consider the following example.

\begin{example}
Assume that $ \mathcal{L}_{ab} $ is the propositional language over $ \Sigma=\{a,b\} $ and $ X=Cn(a\lor b) $. 
Then, $ \neg\belief{a} $ expresses that $ a $ is not part of the belief set, whereas $ \belief{\neg a} $ states that the formula $ \neg a $ is part of the belief set, e.g. $ X \Vdash \neg\belief{a} $ and $ X \not\Vdash \belief{\neg a} $. 
Likewise, $ \belief{a} \lor \belief{b} $ expresses that $ a $ or $ b $ is believed, whereas $ \belief{(a\lor b)} $ states that the formula $ a\lor b $ is believed, e.g. $ X \Vdash \belief{(a\lor b)} $ and $ X \not\Vdash \belief{a}\lor\belief{b} $.
\end{example}

For the setting of belief change, we assume that every agent is equipped with a belief state, also called epistemic state, which contains all information necessary for maintaining her belief apparatus. 
We denote belief states by $ K,K_1,K_2,\ldots $ following the notion of Hansson \cite{KS_Hansson2017}. 
General descriptor revision does not specify what a belief state is, but assumes that a belief set $ \beliefsOf{K} $ is immanent for every epistemic state $ K $.
To make descriptors compatible with belief states, we naturally lift the semantics  to belief states, i.e. $ K \Vdash \Psi $ if $ \beliefsOf{K} \Vdash \Psi $.

\begin{example}[continued%
	]\label{exmpl:descriptor}
		Assume ranking functions as a representation of belief states. 
		Let $ \Kpen $ be the belief state from Table \ref{tbl:exmpl} and let $ \Psi= \{ \belief{\overline{p}},\,\belief{bf},\,\neg\belief{b\overline{f}} \}$ be an \disjunctionFree\ descriptor. 
		The descriptor $ \Psi $ expresses belief in $ \ol{p} $ (it is not a penguin) and $ bf $ (it is a flying bird) and not believing $ b\ol{f} $ (it is a non-flying bird). 
		The immanent belief set of $ \Kpen $ is $ \beliefsOf{\Kpen}=Cn(\ol{p}\land(b\to f)) $. 
		The descriptor $ \Psi $ holds in $ \Kpen $, i.e. $ \Kpen \Vdash \Psi $, since $ \ol{p}\in\beliefsOf{\Kpen} $, $ bf\in\beliefsOf{\Kpen} $ and $ b\ol{f}\notin\beliefsOf{\Kpen} $.
\end{example}

AGM theory \cite{KS_AlchourronGaerdenforsMakinson1985} focuses on properties of revision (or contraction) operations by examining the interconnection between prior belief state, new information and posterior belief state of a change.
Descriptor revision examines the interconnection between prior belief state and posterior belief states that satisfy a particular descriptor.
Let $ \mathbb{K}_K $ denote the set of all reasonable conceivable successor belief states for a belief state $ K $.
A descriptor revision by a descriptor $ \Psi $ is the process of choosing a state $ K' $  from $ \mathbb{K}_K $ such that $ K' \Vdash \Psi  $.
We
abstract from the internal process of how $ \mathbb{K}_K $ is obtained and define descriptor revision\footnote{In the original framework by Hansson this is much more elaborated. 
	By the terminology of Hansson, here we present a form of local deterministic monoselective descriptor revision \cite{KS_Hansson2017}. 
	Moreover, we primarily focus on one change, while Hansson designs the framework for change operators.} as follows.
\begin{definition}[Descriptor Revision, adapted \cite{KS_Hansson2017}]\label{def:descriptor_revision}
Let $ K $ be a belief state, $ \mathbb{K}_K $ a set of belief states and $ C: \mathcal{P}(\mathbb{K}_K) \to \mathbb{K}_K $ be a choice function. 
Then the change from $ K $ to $ K^\circ=K \circ \Psi $ is called a \emph{descriptor revision by $ \Psi $ realised by $ C $ over $ \mathbb{K}_K $} if the following holds:
\begin{equation} \label{eq:descriptor_revision}
K \circ \Psi = C( \, \{ K' \in \mathbb{K}_K \mid K' \Vdash \Psi   \}  \, ),
\end{equation}
\end{definition}
We say that the change from $ K $ to $ K^\circ $ is a descriptor revision (by $ \Psi $), if $ C $ and $ \mathbb{K}_K $ (and $ \Psi $) exist such that the change from $ K $ to $ K^\circ $ is realised by $ C $ over $ \mathbb{K}_K $.
We also say $ K^\circ $ is the result of the descriptor revision of $ K $ (by $ \Psi $ under $ \mathbb{K}_K $).

Descriptors allow to express a variety of different success conditions, e.g. 
\begin{description}
	\item[$ \{\belief{\varphi}\} $]   Revision by $ \varphi $
	\item[$ \{\neg\belief{\varphi}\} $] Contraction by $ \varphi $   (also called revocation \cite{KS_Hansson2019})
	\item[$ \{ \neg\belief{\varphi}, \neg\belief{\neg\varphi} \} $] Giving up the judgement on $ \varphi $ (also called ignoration \cite{BeierleKernIsbernerSauerwaldBockRagni2019KIzeitschrift}) 
\end{description}	
Additionally, Hansson provides the following examples \cite{KS_Hansson2019}:
\begin{description}
	\item[$ \{ \belief{\varphi_1},\ldots,\belief{\varphi_n} \} $] Package revision by $ \{ \varphi_1,\ldots,\varphi_n \} $
	\item[$ \{ \neg\belief{\varphi},\belief{\psi} \} $] Replacement of $ \varphi $ by $ \psi $
	\item[$ \{ \belief{\varphi_1}\lor\ldots\lor\belief{\varphi_n} \} $] Choice revision by $ \{ \varphi_1,\ldots,\varphi_n \} $
	\item[$ \{\belief{\varphi}\lor\belief{\neg\varphi}\} $] Making up one's mind about $ \varphi $
\end{description}
	Note that all given examples, except for choice revision and \enquote{making up one's mind}, are elementary descriptors.
	In particular, elementary descriptor revision subsumes operations of AGM, and, furthermore, also allows to express changes which lead to a revision and a contraction at the same time. For a concrete example, we continue our running example.
\begin{example}[continued%
]\label{exmpl:descriptor_revision}
	Let $ \Kpen $ and $ \Kpcp $ be as in Table \ref{tbl:exmpl}, let $ \mathbb{K}_{\Kpen} $ be the set of all ranking functions, let $ C $ be a choice function such that $ C(X)=\Kpcp $ if $ \Kpcp\in X $, and let $ \Psi=\{ \belief{\ol{b}}\lor\belief{p},\,\neg\belief{bf} \} $ be a descriptor. 
	The descriptor $ \Psi $ expresses posterior belief in $ \ol{b} $ or belief in $ p $ and disbelief in $ bf $. 
	In particular, $ \neg\belief{bf} $ expresses a contraction with $ bf $ (it is a flying bird), but for $ \belief{\ol{b}}\lor\belief{p} $ (it is not a bird or it is a penguin), there is no straight counterpart in the AGM framework.
	Note that we have $ \beliefsOf{\Kpcp}=Cn(\ol{b}\land\ol{p}) $, and thus, it holds that $ \ol{b}\in \beliefsOf{\Kpcp} $ and $ bf \notin \beliefsOf{\Kpcp} $, and therefore, the descriptor $ \Psi $  holds in $ \Kpcp $. Thus, the change from $ \Kpen $ to $ \Kpcp $ is a descriptor revision by $ \Psi $ realised by $ C $ over $ \mathbb{K}_{\Kpen} $.
\end{example}

\section{Conditional Descriptor Revision}
\label{sec:conditional_descriptor_revision}
We instantiate descriptor revision for the case in which the underlying logic is the conditional logic $ \condLang $ and ranking functions serve as a representation for epistemic states.
Furthermore, we adapt the principle of conditional preservation by Kern-Isberner \cite{KS_Kern-Isberner2004} to the requirements of descriptor revision. 

\subsection{Adaptions for Conditionals in $ \condLang $}
\label{sec:descriptors_conditional_descriptor_revision}
In the formal framework of descriptor revision by Hansson, as recalled in Section \ref{sec:prelim_descriptor_revision}, 
semantics of a descriptor refer to a belief set, containing formulas of the underlying logic. 
Thus, when using the advanced logic $ \condLang $, we need to refer to the set of conditionals accepted by a ranking function $ \kappa $ when choosing ranking functions as representations for epistemic states.
However, the belief set $ \beliefsOf{\kappa} $ of a ranking function $ \kappa $ is a set of propositional beliefs, i.e. $ \beliefsOf{\kappa}\subseteq \propLang $. 
We define the set of conditional beliefs for a ranking function $ \kappa $ as follows:
\begin{equation*}
\beliefsOfCond{\kappa} = \{\, \ramseyCond{B}{A} \mid  \kappa \models \ramseyCond{B}{A} \,\}
\end{equation*}
Clearly, the set $ \beliefsOfCond{\kappa} $ is a deductively closed set for every ranking function $ \kappa $ and therefore a belief set.
Descriptors and descriptor revision for $ \condLang $ then refer to the set of conditional beliefs $ \beliefsOfCond{\kappa} $, and their formal definition can be easily obtained by correspondingly modifying Definitions \ref{def:descriptor} to \ref{def:descriptor_revision}.

Note that the conditional logic $ \condLang $ embeds the propositional logic $ \propLang $, hence every proposition $ A\in\propLang $ can be represented by $ \ramseyCond{A}{\top} $. 
Moreover, the definition of $ \beliefsOfCond{\kappa} $ ensures compatibility of propositional beliefs with the conditional beliefs, i.e. $ \{ \ramseyCond{A}{\top} \mid A \in \beliefsOf{K} \} \subseteq \beliefsOfCond{K} $. 
Thus, our approach to descriptor revision by conditionals, presented in the following, subsumes descriptor revision for propositions.

\subsection{Conditional Preservation}
\label{sec:pcp}
When an agent performs a belief change, the change might not only affect explicit beliefs, but also implicit beliefs. 
Boutilier proposed that belief change should also minimize the effect on conditional beliefs \cite{KS_Boutilier1996}. 
Kern-Isberner introduced the principle of conditional preservation (PCP) and gave a thorough axiomatisation of PCP \cite{Kern-Isberner01f,KS_Kern-Isberner2004} in a very general manner. 

Note that the principle of conditional preservation is usually defined as a property of a change by a set of conditionals $ \mathcal{R} $. 
However, when having a descriptor revision, the underlying change framework and its parameters and capabilities might be hidden. 
Thus, we abstract from the assumption that the change is done by a set of conditionals $ \mathcal{R} $, and just state that a change satisfies PCP with respect to a set of conditionals $ \mathcal{R} $.
This allows us to say that a change satisfies the principle of conditional preservation
without assuming 
the involvement of specific parameters
 in the underlying change framework.
In the following, we present our relaxed variant of the principle of conditional preservation for the special case of ranking functions. 

\begin{definition}[PCP for OCF changes, adapted \cite{KS_Kern-IsbernerBockSauerwaldBeierle2017}]\label{def:pcp}
	A change of a ranking function $ \kappa $ to a ranking function $ \kappa^\circ $ fulfils the \emph{principle of conditional preservation
		 with respect to the conditionals} $ \mathcal{R}=\{ \ramseyCond{B_1}{A_1},\ldots,\ramseyCond{B_n}{A_n} \} $, if 
	for every two multisets of propositional interpretations $ \Omega_1=\{\omega_1,\ldots,\omega_m\} $ and $ \Omega_2=\{\omega'_1,\ldots,\omega'_m\} $  with the same cardinality $ m $
	such that the multisets $ \Omega_1 $ and $ \Omega_2 $ contain the same number of interpretations which verify, respectively falsify, each conditional $ \ramseyCond{B_i}{A_i} $ in $ \mathcal{R} $,
 the ranking functions
	$ \kappa $ and $ \kappa^\circ $ 
are	balanced in the following way:
	\begin{equation}\label{eq:pcp_balance}
	\sum_{i=1}^{m} \kappa(\omega_i) - 
	\sum_{i=1}^{m} \kappa(\omega'_i)
	=
	\sum_{i=1}^{m} \kappa^\circ(\omega_i) - 
	\sum_{i=1}^{m} \kappa^\circ(\omega'_i)
	\end{equation}
\end{definition}

\begin{example}[continued%
]\label{exmp:pcp}
Assume our agent has moved to Antarctica and she starts to question her beliefs about penguins and birds.
	The only birds she sees in Antarctica are penguins, and moreover, she observes, trough her window, a lot of penguins jumping off a cliff, and thus, flying for a moment. 
	Her belief state is changing from $ \Kpen $ to $ \Kpcp $ from Table \ref{tbl:exmpl}. 
	Consider now the conditional $ \ramseyCond{p}{b} $ expressing that \emph{birds are usually penguins}, the conditional $ \ramseyCond{f}{p} $ expressing that \emph{penguins usually fly}, and the conditional $ \ramseyCond{\ol{f}}{p} $ expressing that \emph{penguins usually don't fly}. 
	The change from $ \Kpen $ to $ \Kpcp $ satisfies the principle of conditional preservation with respect to the conditionals in $ \mathcal{R}=\{ \ramseyCond{p}{b}, \ramseyCond{f}{p}, \ramseyCond{\ol{f}}{p} \} $. 
	For instance, the two multisets $ \Omega_1=\{ bfp, \ol{b}\,\ol{f}p \} $ and $ \Omega_2=\{ b\ol{f}p, \ol{b}fp  \} $, containing for every conditional in $ \mathcal{R} $ the same number of verifying and falsifying worlds, and their values under $ \Kpen $ and $ \Kpcp $ are balanced according to Equation~\eqref{eq:pcp_balance}, i.e.
	\begin{multline*}
	\Kpen(bfp)+\Kpen(\ol{b}\,\ol{f}p)-\Kpen(b\ol{f}p)-\Kpen(\ol{b}fp)   =  2+2-1-4 = -1\\
	 =1+2-1-3 =\Kpcp(bfp)+\Kpcp(\ol{b}\,\ol{f}p)-\Kpcp(b\ol{f}p)-\Kpcp(\ol{b}fp).
	\end{multline*}
\end{example}

The definition of the principle of conditional preservation, as given in Definition \ref{def:pcp}, does not require information about the success condition of a change.
Thus, the notion of the principle of conditional preservation is directly available for descriptor revision of conditionals when we provide a set of conditionals.
A natural choice are the conditionals appearing in a descriptor $ \Psi $.
For a descriptor $ \Psi $ over $ \condLang $, we define the set of conditionals in $ \Psi $, denoted by $ \mathit{cond}(\Psi) $, as follows:
\begin{itemize}
	\item for $ \Psi=\emptyset $ let $ \mathit{cond}(\Psi)=\emptyset $,
	\item for $ \Psi=\{ \belief{\ramseyCond{B}{A}} \} $ let $ \mathit{cond}(\Psi)=\{ \ramseyCond{B}{A} \} $,
	\item for $ \Psi=\{ \alpha,\beta,\ldots \} $ let $ \mathit{cond}(\Psi)=\mathit{cond}(\{\alpha\})\cup \mathit{cond}(\{\beta,\ldots\}) $,
	\item for $ \Psi=\{ \alpha \lor \beta \} $ let $ \mathit{cond}(\Psi)=\mathit{cond}(\{\alpha\})\cup \mathit{cond}(\{\beta\}) $,
	\item for $ \Psi=\{ \alpha \land \beta \} $ let $ \mathit{cond}(\Psi)=\mathit{cond}(\{\alpha\})\cup \mathit{cond}(\{\beta\}) $, and
	\item for $ \Psi=\{ \neg \alpha \} $ let $ \mathit{cond}(\Psi)=\mathit{cond}(\{\alpha\}) $.
\end{itemize}

In the following, we use a central characterisation \cite{KS_Kern-Isberner2001,KS_Kern-IsbernerBockSauerwaldBeierle2017} of the principle of conditional preservation to obtain a characterisation of the principle of conditional preservation for descriptor revisions.
\begin{proposition}[PCP for Descriptor Revision, adapted \cite{KS_Kern-IsbernerBockSauerwaldBeierle2017}]\label{prop:pcp_characterisation}
	\label{prop:char_c_change}
	Let $ \Psi $ be a descriptor over $ \condLang $ and  $ \mathit{cond}(\Psi) = \{ \, \ramseyCond{B_1}{A_1},\, \ldots\,,\, \ramseyCond{B_n}{A_n} \, \} $
	be the set of conditionals in $ \Psi $, and let $ \kappa^\circ $ be the result of the descriptor revision of $\kappa$ by $ \Psi $. Then this change satisfies the \emph{principle of conditional preservation} with respect to the conditionals in $ \mathit{cond}(\Psi) $ if and only if there are integers\footnote{As noted by Kern-Isberner\cite{KS_Kern-IsbernerBockSauerwaldBeierle2017}, all $ \kappa_0, \gamma_i^+, \gamma_i^- $ can be rational, but $ \kappa^\circ $ has to satisfy the requirements for OCF, in particular, all $ \kappa^\circ(\omega) $ must be non-negative integers.}
	$\kappa_0, \gamma_i^+, \gamma_i^-  \in\mathbb{Z} $, $1 \leq i \leq n$, such that:
	\begin{equation}\label{eq:pcp-characterisation-ocf}
	\kappa^\circ (\omega) = \kappa_0 + \kappa(\omega) + \sum_{1 \leq i \leq n \atop \omega \models A_i B_i} \gamma_i^+ + \sum_{1 \leq i \leq n \atop \omega \models A_i\!\land\!\neg B_i} \gamma_i^-
	\end{equation}
\end{proposition}
The proof of Proposition \ref{prop:char_c_change} is directly obtainable from a proof given by Kern-Isberner \cite[Theorem 4.6.1]{Kern-Isberner00d}, since no specific information on the success condition for the conditionals in the descriptor was used in Proposition \ref{prop:char_c_change}. 
The idea underlying Proposition \ref{prop:char_c_change} is that interpretations that are verifying and falsifying the same conditionals are treated in the same way. 
Thus, for every conditional $ \ramseyCond{B_i}{A_i}\in \mathit{cond}(\Psi) $, the two constants $ \gamma_i^+ $ and $ \gamma_i^- $ handle how interpretations are shifted over the change process. 
The constant $ \kappa_0 $ acts as a normalizer, ensuring that $ \kappa^\circ $ is indeed a ranking function, i.e. there is at least one world $ \omega $ such that $ \kappa^\circ(\omega)=0 $.

\begin{example}[continued%
]\label{exmp:pcp_charactisaion}
Consider the change from $ \Kpen $ to $ \Kpcp $, both given in Table \ref{tbl:exmpl}. 
As shown in Example \ref{exmp:pcp}, this change satisfies the principle of conditional preservation with respect to the conditionals in $ \mathcal{R}=\{ \ramseyCond{p}{b}, \ramseyCond{f}{p}, \ramseyCond{\ol{f}}{p} \} $. 
Indeed, as stated in Proposition \ref{prop:pcp_characterisation}, we can obtain $ \Kpcp $ from $ \Kpen $ via Equation \eqref{eq:pcp-characterisation-ocf} by choosing $ \kappa_0=0$, $ \gamma_1^+=0 $, $ \gamma_1^-=-1 $, $ \gamma_2^+=0 $, $ \gamma_2^-=2 $, $ \gamma_3^+=0 $, and $ \gamma_3^-=0 $.
\end{example}

\subsection{Descriptor Revision with Conditional Preservation}
\label{sec:cdr_instantiation}
\newcommand{\KappaPCP}{\ensuremath{\mathbb{K}^\mathit{PCP}_\kappa}}

The principle of conditional preservation is a powerful basic principle of belief change and it is natural to demand satisfaction of this principle.
The principle demands a specific relation between the conditionals in the prior belief state $ K $, the conditionals in the posterior state $ K^\circ $ and the conditionals in the descriptor $ \Psi $.
Remember that by Definition \ref{def:descriptor_revision}, a descriptor revision from $ K $ to $ K^\circ $ is determined by a choice function $ C $, the descriptor $ \Psi $ and the set $ \mathbb{K}_K $ such that Equation \eqref{eq:descriptor_revision} holds, but none of these components allow to express a direct relation between $ K $, $ K^\circ $ and $ \Psi $. 
Thus, there is no possibility to express conditional preservation by the means of descriptor revision.
The principle of conditional preservation is somewhat orthogonal to descriptor revision, which gives rationale to the following definition of conditional descriptor revision.

\begin{definition}[Conditional Descriptor Revision]\label{def:conditiona_descriptor_revision}
	Let $ \kappa $ be a ranking function.
	A descriptor revision of $ \kappa $ to $ \kappa^\circ $ by a descriptor $ \Psi $ over $ \condLang $ (realised by $ C $ over $ \mathbb{K}_\kappa $) is called a \emph{conditional descriptor revision} of $ \kappa $ to $ \kappa^\circ $ by $ \Psi $ (realised by $ C $ over $ \mathbb{K}_\kappa $) 
	if the change from $ \kappa $ to $ \kappa^\circ $ satisfies the principle of conditional preservation with respect to $ \mathit{cond}(\Psi) $.
\end{definition}
In Definition \ref{def:conditiona_descriptor_revision}, we choose ranking functions as representations for belief states, but note that the principle of conditional preservation also applies to other representations \cite{KS_Kern-Isberner2001}.
Thus, for other kinds of representations of belief states one might give a definition of conditional descriptor revision similar to the one given here.
However, for the rest of the article, we focus on ranking functions. 
Moreover, we assume $ \mathbb{K}_\kappa $ to be the set of all ranking functions, i.e. when revising by a descriptor over $ \Psi $, we choose over the set of all ranking functions.

\begin{example}[continued%
	]\label{exmpl:conditional_descriptor_revision}
	Consider again $ \Kpen $ to $ \Kpsi $ given in Table \ref{tbl:exmpl}.
	The change from $ \Kpen $ to $ \Kpsi $ is a conditional descriptor revision by $ \Psi=\{ \belief{\ramseyCond{p}{b}}, \neg\belief{\ramseyCond{f}{p}}, \neg\belief{\ramseyCond{\overline{f}}{p}} \} $.
	Note that $ \mathit{cond}(\Psi)=\{ \ramseyCond{p}{b}, \ramseyCond{f}{p}, \ramseyCond{\overline{f}}{p} \} $, and therefore, as stated in Example \ref{exmp:pcp}, the change from $ \Kpen $ to $ \Kpsi $ satisfies the principle of conditional preservation with respect to $ \mathit{cond}(\Psi) $. 
	Note that $ \Psi $ holds in $ \Kpcp $, i.e. $ \Kpcp \Vdash \Psi $. 
	In particular, it is the case that $ \Kpcp \Vdash \neg\belief{\ramseyCond{\overline{f}}{p}} $, which is equivalent to $ \Kpcp\not\models \ramseyCond{\overline{f}}{p} $, i.e. $ \Kpcp(\overline{f}p) \not< \Kpcp(fp) $.
\end{example}

\section{Characterisation of Conditional Descriptor Revision with Elementary Descriptors by CSPs}
\label{sec:constraint_system}
\newcommand{\crDescriptor}[1]{\ensuremath{\mathit{CR}_{\!\mathit{D}}(#1)}}
\newcommand{\solCrDescriptor}[1]{\ensuremath{\mathit{Sol}(\mathit{CR}_{\!\mathit{D}}(#1))}}

The arithmetic nature of ranking functions and the characterisation of the principle of conditional preservation by Proposition \ref{prop:pcp_characterisation} allow us to give a constraint, expressing the success condition of a literal descriptor.

\begin{definition}[Constraint for literal descriptors, $ \crDescriptor{\kappa,\alpha,\Psi} $]\label{def:constract_literal_descriptor}
Let $ \kappa $ be a ranking function, let $ \Psi=\{\alpha_1,\ldots,\alpha_m\} $ an \disjunctionFree\ descriptor over $ \condLang $ with $ \mathit{cond}(\Psi)=\{ \ramseyCond{A_1}{B_1},\ldots, \ramseyCond{A_n}{B_n} \} $, and let $ \alpha $ be a literal descriptor in $ \Psi $. 
The constraint for $ \alpha $ in $ \kappa $ under $ \Psi $, denoted by $ \crDescriptor{\kappa,\alpha,\Psi} $, on the constraint variables $ \gamma_1^+,\gamma_1^-,\ldots,\gamma_n^+,\gamma_n^- $ ranging over $ \mathbb{Z} $, is given for a positive literal $ \alpha=\belief{\ramseyCond{B_i}{A_i}} $ descriptor by 
	\begin{equation}\label{eq:pos_literal_descriptor}
	\begin{split}
\gamma_i^- - \gamma_i^+ > & \, (\min_{\omega \vDash A_i B_i} \kappa(\omega) + \sum_{\substack{j \neq i \\ \omega \vDash A_j B_j}} \gamma_j^+ + \sum_{\substack{j \neq i \\ \omega \vDash A_j \bar{B}_j}} \gamma_j^-) \\
- & (\min_{\omega \vDash A_i \bar{B}_i} \kappa(\omega) + \sum_{\substack{j \neq i \\ \omega \vDash A_j B_j}} \gamma_j^+ + \sum_{\substack{j \neq i \\ \omega \vDash A_j \bar{B}_j}} \gamma_j^-) \quad \text{ for $i = 1, \dots, n$} 
\end{split}
\end{equation}
and for a negative literal descriptor $ \alpha=\neg\belief{\ramseyCond{B_i}{A_i}} $ by
\begin{equation}\label{eq:neg_literal_descriptor}
\begin{split}
\gamma_i^- - \gamma_i^+ \leq &\,  (\min_{\omega \vDash A_i B_i} \kappa(\omega) + \sum_{\substack{j \neq i \\ \omega \vDash A_j B_j}} \gamma_j^+ + \sum_{\substack{j \neq i \\ \omega \vDash A_j \bar{B}_j}} \gamma_j^-) \\
- & (\min_{\omega \vDash A_i \bar{B}_i} \kappa(\omega) + \sum_{\substack{j \neq i \\ \omega \vDash A_j B_j}} \gamma_j^+ + \sum_{\substack{j \neq i \\ \omega \vDash A_j \bar{B}_j}} \gamma_j^-) \quad \text{ for $i = 1, \dots, n$}.
\end{split}
\end{equation}
\end{definition}
The rationale for Definition \ref{def:constract_literal_descriptor} is that a positive literal descriptor $ \{ \belief{\ramseyCond{B}{A}} \} $ holds in the posterior state $ \kappa^\circ $ if $ \ramseyCond{B}{A} $ is accepted by $ \kappa^\circ $, more formally $ \kappa^\circ \models \ramseyCond{B}{A} $, i.e. $ \kappa^\circ(AB) < \kappa^\circ(A\ol{B}) $. Likewise, a negative literal descriptor $ \{\neg\belief{\ramseyCond{B}{A}}\} $ corresponds to $ \kappa^\circ \not\models \ramseyCond{B}{A} $, i.e. $ \kappa^\circ(AB) \geq  \kappa^\circ(A\ol{B}) $.
The combining of all the constraints obtained for each literal descriptor in $ \Psi $ yields a constraint satisfaction problem.

\begin{definition}[CSP for \disjunctionFree\  descriptors, $ \crDescriptor{\kappa,\Psi} $]\label{def:constract_conditional_descriptor_revision}
	Let $ \kappa $ be a ranking function and $ \Psi $ be an \disjunctionFree\  belief descriptor with $ \mathit{cond}(\Psi)=\{ \ramseyCond{A_1}{B_1},\ldots, \ramseyCond{A_n}{B_n} \} $.
	The constraint satisfaction problem for $ \kappa $ and $ \Psi $, on the constraint variables $ \gamma_1^+,\gamma_1^-,\ldots,\gamma_n^+,\gamma_n^- $ ranging over $ \mathbb{Z} $, denoted by $ \crDescriptor{\kappa,\Psi} $, is given by the conjunction of %
	the constraints $ \crDescriptor{\kappa,\alpha,\Psi} $ for each $ \alpha \in \Psi $.	
\end{definition}

With $ \solCrDescriptor{\kappa,\Psi} $, we denote the solutions of the constraint satisfaction problem $ \crDescriptor{\kappa,\Psi} $. Each solution $ \vec{\gamma} = \tuple{\gamma_1^+,\gamma_1^-,\ldots,\gamma_n^+,\gamma_n^-} \in \solCrDescriptor{\kappa,\Psi} $ induces a unique ranking function $ \kappa_{\vec{\gamma}} $ obtained from Equation \eqref{eq:pcp-characterisation-ocf} in Theorem \ref{prop:char_c_change} by choosing $ \kappa_0 $ as the smallest integer  such that the equation yields a ranking function, i.e., there is a propositional interpretation $ \omega\in\Omega $ such that $ \kappa_{\vec{\gamma}}(\omega)=0 $ and for all $ \omega\in\Omega $ the value $ \kappa_{\vec{\gamma}}(\omega) $ is a non-negative integer.

\begin{example}[continued%
	]\label{exmpl:constraint_system_for_descriptor_revision}
Consider $ \Kpen $ from Table \ref{tbl:exmpl} and the \disjunctionFree\ descriptor $ \Psi=\{ \belief{\ramseyCond{p}{b}}, \neg\belief{\ramseyCond{f}{p}}, \neg\belief{\ramseyCond{\overline{f}}{p}} \} $. The  CSP $ \crDescriptor{\kappa,\Psi} $ is given by:
\begin{align*}
\crDescriptor{\Kpen,\belief{\ramseyCond{p}{b}},\Psi}\!:
& &
\gamma_1^- - \gamma_1^+ >  &  \min\{ \Kpen(bfp)  + \gamma_2^+ {+} \gamma_3^- ,\, \Kpen(b\overline{f}p) + \gamma_3^+ + \gamma_2^- \} \\
& &   &  - \min\{ \Kpen(bf\overline{p}), \Kpen(b\overline{f}\overline{p}))  \} \\[0.5em]
\crDescriptor{\Kpen,\neg\belief{\ramseyCond{f}{p}},\Psi}\!:
& &
\gamma_2^- - \gamma_2^+ \leq  &  \min\{ \Kpen(bfp) + \gamma_1^+ + \gamma_3^- ,\, \Kpen(\overline{b}fp) + \gamma_3^- \} \\
& &   &  - \min\{ \Kpen(b\overline{f}p) + \gamma_1^+ + \gamma_3^+, \Kpen(\overline{b}\,\overline{f}p))  + \gamma_3^+  \} \\[0.5em]
\crDescriptor{\Kpen,\neg\belief{\ramseyCond{\overline{f}}{p}},\Psi}\!:
& &
\gamma_3^- - \gamma_3^+ \leq  &  \min\{ \Kpen(b\overline{f}p) + \gamma_1^+ + \gamma_2^- ,\, \Kpen(\overline{b}\,\overline{f}p) + \gamma_2^- \} \\
& &   &  - \min\{ \Kpen(bfp) + \gamma_1^+ + \gamma_2^+, \Kpen(\overline{b}fp))  + \gamma_2^+  \} 
\end{align*}
The vector $ \vec{\gamma}=\tuple{\gamma_1^+,\gamma_1^-,\gamma_2^+,\gamma_2^-,\gamma_3^+,\gamma_3^-} $ with $ \gamma_1^+=0 $, $ \gamma_1^-=-1 $, $ \gamma_2^+=0 $, $ \gamma_2^-=2 $, $ \gamma_3^+=0 $, and $ \gamma_3^-=0 $ is a solution of $ \solCrDescriptor{\Kpen,\Psi} $, i.e. $ \vec{\gamma} \in\solCrDescriptor{\Kpen,\Psi} $.
We obtain the ranking function $ \Kpcp=\kappa_{\vec{\gamma}} $ given in Table \ref{tbl:exmpl}.
\end{example}

We examine whether our approach is sound and complete with respect to conditional descriptor revision.

\begin{theorem}[Soundness of $ \crDescriptor{\kappa,\Psi} $]\label{prop:soundness_descriptor}
Let $ \kappa $ be an ordinal conditional ranking function, $ \Psi $ be an \disjunctionFree\  belief descriptor, and let $ \vec{\gamma} \in \solCrDescriptor{\Psi} $.  Then, the change from $ \kappa $ to $ \kappa_{\vec{\gamma}} $ is a conditional descriptor revision by $ \Psi $ (over all ranking functions).
\end{theorem}
Note that a ranking function $ \kappa^\circ $ is a c-representation \cite{Kern-Isberner00d} for a set of conditionals $ \mathcal{R} $ if and only if $ \kappa^\circ $ is the result of a conditional descriptor revision starting form a ranking function $ \kappa $ such that $  \kappa(\omega)=0$ for every $ \omega\in\Omega $ with a descriptor $ \Psi=\{ \belief{\ramseyCond{B}{A}}  \mid \ramseyCond{B}{A} \in \mathcal{R}  \} $. 
The construction of a c-representation can be characterised by a constraint-satisfaction problem similar to the one given in Definition \ref{def:constract_conditional_descriptor_revision} \cite{Kern-Isberner00d,BeierleEichhornKernIsbernerKutsch2018AMAI}. 
The soundness proof transfers to a proof of Theorem \ref{prop:soundness_descriptor}.

\begin{theorem}[Completeness of $ \crDescriptor{\kappa,\Psi} $]\label{prop:completeness_descriptor}
	Let $ \Psi $ be an \disjunctionFree\  belief descriptor and $ \kappa,\kappa^\circ $ be ordinal conditional functions.
	If the change from $ \kappa $ to $ \kappa^\circ $ is a conditional descriptor revision by $ \Psi $ (over all ranking functions), then there exists an vector $ \vec{\gamma}\in\solCrDescriptor{\kappa,\Psi} $ such that $ \kappa^\circ=\kappa_{\vec{\gamma}} $.
\end{theorem}
\begin{proof}[sketch] 
	Because of Proposition \ref{prop:pcp_characterisation}, there exists $ \kappa_0 $ and $ \vec{\gamma}=\tuple{\gamma_1^+,\gamma_1^-,\ldots} $ such that the ranking function $ \kappa^\circ $  is representable as stated in Equation \eqref{eq:pcp-characterisation-ocf}. 
	Therefore, we have $ \kappa^\circ=\kappa_{\vec{\gamma}} $. 
	It remains to show that $ \vec{\gamma}\in\solCrDescriptor{\kappa,\Psi} $. 
	Note that by our assumptions
	$ \kappa^\circ\Vdash\alpha $ holds for each $ \alpha\in\Psi $.
	Suppose that $ \alpha $ is a positive literal descriptor, i.e. $ \alpha=\belief{\ramseyCond{B}{A}} $, and thus, $ \kappa^\circ(AB) < \kappa^\circ(A\ol{B}) $. 
	By employing Equation \eqref{eq:pcp-characterisation-ocf}, we obtain Equation \eqref{eq:pos_literal_descriptor} from $ \kappa^\circ(AB) < \kappa^\circ(A\ol{B}) $  by algebraic transformations \cite{Kern-Isberner00d}.
	In an analogue way, one can obtain Equation \eqref{eq:neg_literal_descriptor} from a negative literal descriptor. 
	Note that these are exactly the inequalities in $ \crDescriptor{\kappa,\Psi} $. 
	Therefore, the vector $ \vec{\gamma} $ is a solution for $ \solCrDescriptor{\kappa,\Psi} $.
\end{proof}

\section{Implementation by ChangeOCF}
\label{sec:implementation}

\newcommand{\crDescriptorMaxImpactLimit}[1]{\ensuremath{\mathit{CR}^{\vec{u}}_{\!\mathit{D}}(#1)}}

We implemented descriptor revision for conditionals and elementary descriptors under the principle of conditional preservation.
Given a ranking function $\kappa$ and an elementary descriptor $\Psi$, our system, called ChangeOCF, calculates a list of possible outcomes of a revision of $\kappa$ with $\Psi$.
To calculate the possible outcomes of the revision, ChangeOCF uses a constraint system based on $\crDescriptor{\kappa,\Psi}$ introduced in Section~\ref{sec:constraint_system}.
Following the Propositions~\ref{prop:soundness_descriptor} and \ref{prop:completeness_descriptor}, the solutions of this constraint system correspond to the outcomes of a conditional descriptor revision.
A straightforward approach would be to solve $\crDescriptor{\kappa,\Psi}$ for the given $\kappa$ and $\Psi$.
Then, for each $\vec{\gamma} \in \solCrDescriptor{\Psi}$ the corresponding ranking function $\kappa_{\vec{\gamma}}$ is calculated.

In general, $\solCrDescriptor{\Psi}$ may contain infinite elements, but there is only a finite number of equivalence classes with respect to the acceptance of conditionals.
Therefore, it is possible to restrict the set of solutions to finitely many without losing interesting results.
To do this, we used an approach inspired by \emph{maximal impacts} for c-representations \cite{BeierleEichhornKernIsbernerKutsch2018AMAI} that addresses a similar problem for the enumeration of c-representations.
The idea of maximal impacts is to add explicit bounds for the value of each $\gamma_i^+, \gamma_i^-$.
This reduces the set of possible solutions to a finite set, without losing equivalent solutions when choosing the bounds appropriately.
ChangeOCF limits the value of $\gamma_1^+, \gamma_1^-, \dots, \gamma_n^+, \gamma_n^-$ to an individual finite domain by extending
the constraint system $\crDescriptor{\kappa,\Psi}$ with constraints $ u_i^{\min-} \leq \gamma_i^- \leq u_i^{\max-} $ and $ u_i^{\min+} \leq \gamma_i^+ \leq u_i^{\max+} $ for $ 1\leq  i \leq n $.
We denote this extended constraint system by $\crDescriptorMaxImpactLimit{\kappa,\Psi}$ with $\vec{u} = \langle u_1^{\min-},\allowbreak u_1^{\max-},\allowbreak u_1^{\min+},\allowbreak u_1^{\max+},\allowbreak \dots,  u_n^{\max+} \rangle$.
Like for c-representations \cite{KomoBeierle2020ISAIM}, it is an open problem 
	which values for $ \vec{u} $ guarantee that a representative for each equivalence class of solutions with respect to the acceptance of conditionals is found for a given $\kappa$ and $\Psi$.

The implementation of
ChangeOCF is 
build upon
by InfOCF-Lib\cite{Kutsch2019DKBKIK}, a Java library for reasoning with conditionals and ranking functions. 
InfOCF-Lib calculates the c-representations of a conditional knowledge base by solving a constraint system similar to $\crDescriptorMaxImpactLimit{\kappa,\Psi}$.
The interface of ChangeOCF  is implemented in Java.
To solve $\crDescriptorMaxImpactLimit{\kappa,\Psi}$, we use SICStus Prolog and its constraint logic programming library for finite domains \cite{CarlssonOttossonCarlson97sicstusCLPfd}.
The Prolog implementation is an adaption of the implementation of InfOCF \cite{BeierleEichhornKutsch2017KIzeitschrift} to the more general case of belief change. 
\begin{example}[continued%
	]\label{exmpl:drc_implementation}
	Consider again the descriptor revision of $ \Kpen $ from Table \ref{tbl:exmpl} with the \disjunctionFree\ descriptor $ \Psi=\{ \belief{\ramseyCond{p}{b}}, \neg\belief{\ramseyCond{f}{p}}, \neg\belief{\ramseyCond{\overline{f}}{p}} \} $.
	The corresponding constraint satisfaction problem $ \crDescriptorMaxImpactLimit{\kappa,\Psi} $ is given by the conjunction of \crDescriptor{\kappa,\Psi} from Example \ref{exmpl:constraint_system_for_descriptor_revision} with the following constraints:
\begin{align*}
	   u_1^{\min-} &\leq \gamma_1^-  \leq u_1^{\max-}		& u_2^{\min-} &\leq \gamma_2^-  \leq u_1^{\max-}	& u_3^{\min-} &\leq \gamma_3^-  \leq u_3^{\max-}\\
	   u_1^{\min+} &\leq \gamma_1^+  \leq u_1^{\max+} 	 & u_2^{\min+} &\leq \gamma_2^+  \leq u_1^{\max+}	& u_3^{\min+} &\leq \gamma_3^+  \leq u_3^{\max+}
	\end{align*}
	If we choose for example $\vec{u} = \langle -2, 0, 0, 2, -1, 1, -1, 1, 0, 0, 0, 0\rangle$, 
	there are nine solutions to $ \crDescriptorMaxImpactLimit{\kappa,\Psi} $.
 	One of the solutions is $\vec{\gamma} = \langle 0, 2, -1, 0, 0, 0 \rangle$, which corresponds to $\induzierteOCF{\vec{\gamma}} = \Kpcp$ from Table \ref{tbl:exmpl}.

\end{example}

\section{Summary and Future Work}
\label{sec:conclusion}
In this article, we investigated descriptor revision for a conditional logic and its realisation.
We defined elementary descriptors, a large fragment of the full descriptor language, allowing to express a multitude of different kind of changes processes. 
In particular, elementary descriptors cover the success conditions of AGM revision and AGM contraction.
We introduced conditional descriptor revision, which is an extension of descriptor revision for conditionals obeying the principle of conditional preservation by Kern-Isberner.
We gave a characterisation by a constraint satisfaction problem and an implementation of conditional descriptor revision with elementary descriptors was presented.

For future work, we plan to give a characterisation of conditional descriptor revision with descriptors with disjunction. 
This requires a more fine-grained handling of the interaction of the constraints, and might require transformations of a descriptor into a normal form. 
Another open problem is the determination of maximal impacts for the constraint problem such that all solutions up to equivalence with respect to acceptance of conditionals are captured. 

\clearpage
\bibliographystyle{splncs04}

\newcommand{\verzeichnisBibtex}{\string~/BibTeXReferencesSVNlink}
\bibliography{bibexport}%

\end{document}